\documentclass{aamas2017}

\usepackage{amssymb}
\usepackage{amsmath}

\usepackage{amsthm}
\usepackage{bbm}
\usepackage{graphicx}
\usepackage{subfigure}
\usepackage{url}
\usepackage{threeparttable}
\usepackage{multirow}
\usepackage[table,xcdraw]{xcolor}
\usepackage[linesnumbered,vlined,ruled]{algorithm2e}
\SetKwProg{Fn}{Function}{}{}
\usepackage{tabularx, booktabs}

\theoremstyle{definition}
\newtheorem{thm}{Theorem}
\newtheorem{pro}[thm]{Property}

\newtheorem{defn}{Definition}
\newcommand{\hang}[1]{#1}

\renewcommand\arraystretch{0.6}

\begin{document}

\title{Lifelong Multi-Agent Path Finding \\ for Online Pickup and Delivery
  Tasks\thanks{\small Our research was supported by NSF under grant numbers
    1409987 and 1319966. The views and conclusions contained in this document
    are those of the authors and should not be interpreted as representing the
    official policies, either expressed or implied, of the sponsoring
    organizations, agencies or the U.S. government.}}

\numberofauthors{4}

\author{
\alignauthor Hang Ma\\
  \affaddr{University of Southern California}\\
  \email{hangma@usc.edu}
\alignauthor Jiaoyang Li\\
  \affaddr{Tsinghua University}\\
  \email{lijiaoyang13@mails.tsinghua.edu.cn}
  \and
\alignauthor  T.~K.~Satish Kumar\\
  \affaddr{University of Southern California}\\
  \email{tkskwork@gmail.com}
\alignauthor  Sven Koenig\\
  \affaddr{University of Southern California}\\
  \email{skoenig@usc.edu}
}

\maketitle

\begin{abstract}
The multi-agent path-finding (MAPF) problem has recently received a lot of
attention. However, it does not capture important characteristics of many
real-world domains, such as automated warehouses, where agents are constantly
engaged with new tasks. In this paper, we therefore study a lifelong version
of the MAPF problem, called the multi-agent pickup and delivery (MAPD)
problem. In the MAPD problem, agents have to attend to a stream of delivery
tasks in an online setting. One agent has to be assigned to each delivery
task. This agent has to first move to a given pickup location and then to a
given delivery location while avoiding collisions with other agents. We
present two decoupled MAPD algorithms, Token Passing (TP) and Token Passing
with Task Swaps (TPTS). Theoretically, we show that they solve all well-formed
MAPD instances, a realistic subclass of MAPD instances.  Experimentally, we
compare them against a centralized strawman MAPD algorithm without this
guarantee in a simulated warehouse system. TP can easily be extended to a
fully distributed MAPD algorithm and is the best choice when real-time
computation is of primary concern since it remains efficient for MAPD
instances with hundreds of agents and tasks. TPTS requires limited
communication among agents and balances well between TP and the centralized
MAPD algorithm.
\end{abstract}








\section{Introduction}

Many real-world applications of multi-agent systems require agents to operate
in known common environments. The agents are constantly engaged with new tasks
and have to navigate between locations where the tasks need to be executed.
Examples include aircraft-towing vehicles~\cite{airporttug16}, warehouse
robots~\cite{kiva}, office robots~\cite{DBLP:conf/ijcai/VelosoBCR15}, and game
characters in video games~\cite{WHCA}. In the near future, for instance,
aircraft-towing vehicles might navigate autonomously to aircraft and tow them
from the runways to their gates so as to reduce pollution, energy consumption,
congestion, and human workload. Today, warehouse robots already navigate
autonomously to inventory pods and move them from their storage locations to
packing stations.

Past research efforts have concentrated mostly on a ``one-shot'' version of
this problem, called the multi-agent path-finding (MAPF) problem, which has
been studied in artificial intelligence, robotics, and operations research. In
the MAPF problem, each agent has to move from its current location to its
destination while avoiding collisions with other agents in a known common
environment. The number of agents is the same as the number of destinations,
and the MAPF task ends once all agents reach their destinations. Therefore,
the MAPF problem does not capture important characteristics of many real-world
domains, such as automated warehouses, where agents are constantly engaged
with new tasks.

In this paper, we therefore study a ``lifelong'' version of the MAPF problem,
called the multi-agent pickup and delivery (MAPD) problem. In the MAPD
problem, agents have to attend to a stream of delivery tasks in a known common
environment that is modeled as an undirected graph. Tasks can enter the system
at any time and are modeled as exogenous events that are characterized by a
pickup location and a delivery location each. An agent that is currently not
executing any task can be assigned to an unexecuted task. In order to execute
the task, the agent has to first move from its current location to the pickup
location and then to the delivery location of the task while avoiding
collisions with other agents. We first formalize the MAPD problem and then
present two decoupled MAPD algorithms, Token Passing (TP) and Token Passing
with Task Swaps (TPTS), both of which are based on existing MAPF
algorithms. Theoretically, we show that they solve all well-formed MAPD
instances \cite{CapVK15}, a realistic subclass of MAPD instances.
Experimentally, we compare them against a centralized strawman MAPD algorithm
without this guarantee in a simulated warehouse system.

\section{Background and Related Work}

The MAPD problem requires both the assignment of agents to tasks in an online
and lifelong setting and the planning of collision-free paths. In a lifelong
setting, agents have to attend to a stream of tasks. Therefore, agents cannot
rest in their destinations after they finish executing tasks. In an online
setting, tasks can enter the system at any time. Therefore, assigning agents
to tasks and path planning cannot be done in advance but rather need to be
done during execution in real-time.

The decentralized assignment of agents to more than one task each has been
studied before in isolation \cite{Tovey2005,ZhengIJCAI,AAAI15-MacAlpine}. The
decentralized planning of collision-free paths has also been studied before in
isolation, including with reactive approaches~\cite{ORCA} and prioritized
approaches \cite{ErdmannL87}, but these approaches can result in
deadlocks. The planning of collision-free paths has also been studied in the
context of the MAPF problem, which is a one-shot (as opposed to a lifelong)
version of the MAPD problem.  It is NP-hard to solve optimally for minimizing
flowtime (the sum of the number of timesteps required by all agents to reach
their destinations and stop moving) and NP-hard to approximate within any
constant factor less than 4/3 for minimizing makespan (the timestep when all
agents have reached their destinations and stop moving)~\cite{MaAAAI16}. It
can be solved via reductions to Boolean Satisfiability~\cite{Surynek15},
Integer Linear Programming~\cite{YuLav13ICRA}, and Answer Set
Programming~\cite{erdem2013general}. Optimal dedicated MAPF algorithms include
Independence Detection with Operator Decomposition~\cite{ODA11}, Enhanced
Partial Expansion A*~\cite{EPEJAIR}, Increasing Cost Tree
Search~\cite{DBLP:journals/ai/SharonSGF13}, M*~\cite{wagner15}, and
Conflict-Based
Search~\cite{DBLP:journals/ai/SharonSFS15,ICBS,CohenUK16}. Suboptimal
dedicated MAPF algorithms include Windowed-Hierarchical Cooperative
A*~\cite{WHCA,WHCA06}, Push and Swap/Rotate~\cite{PushAndSwap,PushAndRotate},
TASS~\cite{KhorshidHS11}, BIBOX~\cite{Surynek09}, and MAPP~\cite{WangB11}. The
MAPF problem has recently been generalized to more clearly resemble real-world
settings \cite{HoenigICAPS16,MaWOMPF16,HoenigIROS16,MaAAAI17,MaAAMAS16} but
these versions are still one-shot.

\section{Problem Definition}

In this section, we first formalize the MAPD problem and then define
well-formed MAPD instances.

\subsection{MAPD Problem}

An instance of the MAPD problem consists of $m$ agents $A = \{a_1, a_2 \ldots
a_m\}$ and an undirected connected graph $G = (V,E)$ whose vertices $V$
correspond to locations and whose edges $E$ correspond to connections between
locations that the agents can move along. Let $l_i(t) \in V$ denote the
location of agent $a_i$ in discrete timestep $t$. Agent $a_i$ starts in its
initial location $l_i(0)$. In each timestep $t$, the agent either stays in its
current location $l_i(t)$ or moves to an adjacent location, that is, $l_i(t+1)
= l_i(t)$ or $(l_i(t), l_i(t+1)) \in E$. Agents need to avoid collisions with
each other: (1) Two agents cannot be in the same location in the same
timestep, that is, for all agents $a_i$ and $a_{i'}$ with $a_i \neq a_{i'}$
and timesteps $t$: $l_i(t) \neq l_{i'}(t)$; and (2) two agents cannot move
along the same edge in opposite directions in the same timestep, that is, for
all agents $a_i$ and $a_{i'}$ with $a_i \neq a_{i'}$ and all timesteps $t$:
$l_i(t) \neq l_{i'}(t+1)$ or $l_{i'}(t) \neq l_i(t+1)$. \hang{A path is a
  sequence of locations with associated timesteps, that is, a mapping from an
  interval of timesteps to locations. Two paths collide iff the two agents
  that move along them collide.}

Consider a task set $\mathcal{T}$ that contains the set of unexecuted
tasks. In each timestep, the system adds all new tasks to the task set. Each
task $\tau_j \in \mathcal{T}$ is characterized by a pickup location $s_j \in
V$ and a delivery location $g_j \in V$. An agent is called free iff it is
currently not executing any task. Otherwise, it is called occupied. A free
agent can be assigned to any task $\tau_j \in \mathcal{T}$. It then has to
move from its current location via the pickup location $s_j$ of the task to
the delivery location $g_j$ of the task. (Any agent that had been assigned to
this task previously no longer has this obligation.) When the agent reaches
the pickup location, it starts to execute the task and the task is removed
from $\mathcal{T}$. When it reaches the delivery location, it finishes
executing the task, which implies that it becomes free again and is no longer
assigned to the task. Note that any free agent can be assigned to any task in
the task set. An agent can be assigned to a different task in the task set
while it is still moving to the pickup location of the task it is currently
assigned to but it has first to finish executing the task after it has reached
its pickup location.  These properties model delivery tasks, where agents can
often be re-tasked before they have picked up a good but have to deliver it
afterward.

The objective is to finish executing each task as quickly as possible.
Consequently, the effectiveness of a MAPD algorithm is evaluated by the
average number of timesteps, called service time, needed to finish executing
each task after it was added to the task set. A MAPD algorithm solves a MAPD
instance iff the resulting service time of all tasks is bounded.

\subsection{Well-Formed MAPD Instances}

\begin{figure}
  \centering
  \includegraphics[width=0.3\columnwidth]{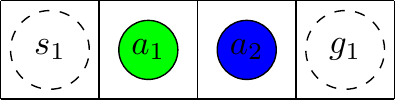}
  \caption{The figure shows a MAPD instance with two free agents $a_1$ and
    $a_2$ and one task $\tau_1$ with pickup location $s_1$ and delivery
    location $g_1$.}
  \label{fig:env}
\end{figure}

Not every MAPD instance is solvable. Figure~\ref{fig:env} shows an example
with two free agents $a_1$ and $a_2$ where neither agent can finish executing
task $\tau_1$ with pickup location $s_1$ and delivery location $g_1$. We now
provide a sufficient condition that makes MAPD instances solvable, namely
being well-formed \cite{CapVK15}. The intuition is that agents should only be
allowed to rest (that is, stay forever) in locations, called endpoints, where
they cannot block other agents. For example, office workspaces are typically
placed in office environments so as not to block routes. The set $V_{ep}$ of
endpoints of a MAPD instance contains all initial locations of agents, all
pickup and delivery locations of tasks, and perhaps additional designated
parking locations. Let $V_{tsk}$ denote the set of all possible pickup and
delivery locations of tasks, called the task endpoints. The set $V_{ep}
\setminus V_{tsk}$ is called the set of non-task endpoints.

\begin{figure}
  \centering
  \includegraphics[width=0.25\columnwidth]{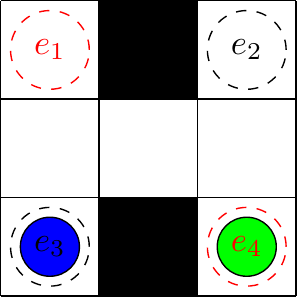}~~
  \includegraphics[width=0.25\columnwidth]{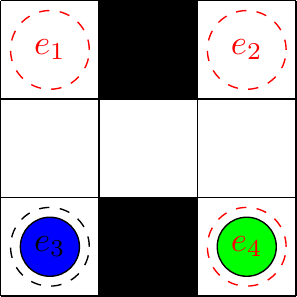}~~
  \includegraphics[width=0.25\columnwidth]{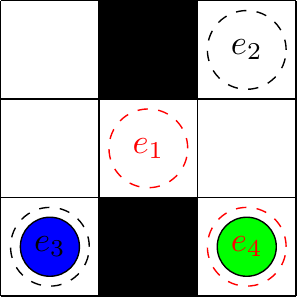}
  \caption{The figure shows three MAPD instances. Black cells are
    blocked. Blue and green circles are the initial locations of agents. Red
    dashed circles are task endpoints. Black dashed circles are non-task
    endpoints.}
  \label{fig:well-formed}
\end{figure}

\begin{defn}\label{defn:well-formed}
  A MAPD instance is well-formed iff a) the number of tasks is finite, b)
  there are no fewer non-task endpoints than the number of agents, and c) for
  any two endpoints, there exists a path between them that traverses no other
  endpoints.
\end{defn}

Well-formed MAPD instances (with at least one task) have at least $m+1$
endpoints. Figure~\ref{fig:well-formed} shows three MAPD instances. The MAPD
instance on the left is well-formed. The MAPD instance in the center is not
well-formed because there are two agents but only one non-task endpoint. The
MAPD instance on the right is not well-formed because, for example, all paths
between endpoints $e_2$ and $e_3$ traverse endpoint $e_1$. We design two
decoupled MAPD algorithms in the following that both solve all well-formed
MAPD instances (even though they might not execute all tasks in case the
number of tasks is infinite).

\section{Decoupled MAPD Algorithms}

In this section, we present first a simple decoupled MAPD algorithm, called
Token Passing (TP), and then an improved version, called Token Passing with
Task Swaps (TPTS), that is more effective. Decoupled MAPD algorithms are those
where each agent assigns itself to tasks and computes its own collision-free
paths given some global information.

\subsection{Token Passing (TP)}

Token Passing (TP) is based on an idea similar to Cooperative A*~\cite{WHCA},
where agents plan their paths one after the other. Its task set contains all
tasks that have no agents assigned to them. We describe a version of TP that
uses token passing and can thus easily be extended to a fully distributed MAPD
algorithm. The token is a synchronized shared block of memory that contains
the current paths of all agents, task set, and agent assignments. \hang{All
  MAPD algorithms in this paper, including TP, always assume that an agent
  rests (that is, stays forever) in the last location of its path in the token
  when it reaches the end of it.} Token passing has previously been used to
develop COBRA \cite{CapVK15}, which is a MAPF-like algorithm that does not
take into account that pickup or delivery locations of tasks can be occupied
by agents not executing them and can thus result in deadlocks.

\begin{algorithm}[t]
\scriptsize
\renewcommand\arraystretch{0.5}
\caption{Token Passing (TP)}
\label{alg:TP}
/* system executes now */\;
Initialize \emph{token} with the (trivial) path $[loc(a_i)]$ for each agent $a_i$\;
\While{true}
{
    Add all new tasks, if any, to the task set $\mathcal{T}$\;
    \While{\textnormal{agent $a_i$ exists that requests \emph{token}}}
    {
      /* system sends \emph{token} to $a_i$ - $a_i$ executes now */\;
      $\mathcal{T'} \gets \{\tau_j \in \mathcal{T}| $no other path in \emph{token} ends in $s_j$ or $g_j$$\}$\;
      \If{$\mathcal{T'}\neq \emptyset$}
         {
           $\tau \gets \arg \min_{\tau_j \in \mathcal{T'}} h(loc(a_i), s_j)$\;
           Assign $a_i$ to $\tau$\;
           Remove $\tau$ from $\mathcal{T}$\;
           Update $a_i$'s path in \emph{token} with Path1($a_i$, $\tau$, \emph{token})\;
         }
         \ElseIf{\textnormal{no task $\tau_j \in \mathcal{T}$ exists with $g_j = loc(a_i)$}}
                {
                  Update $a_i$'s path in \emph{token} with the path $[loc(a_i)]$\;
                }
                \Else
                    {
                      Update $a_i$'s path in \emph{token} with Path2($a_i$, \emph{token})\;
                    }
                    /* $a_i$ returns \emph{token} to system - system executes now */\;
    }
    All agents move along their paths in \emph{token} for one timestep\;
    /* system advances to the next timestep */\;
}
\end{algorithm}

Algorithm~\ref{alg:TP} shows the pseudo-code of TP, where $loc(a_i)$ denotes
the current location of agent $a_i$. Agent $a_i$ finds all paths via A*
searches in a state space whose states are pairs of locations and timesteps. A
directed edge exists from state $(l,t)$ to state $(l',t+1)$ iff $l = l'$ or
$(l,l') \in E$. State $(l,t)$ is removed from the state space iff $a_i$ being
in location $l$ at timestep $t$ results in it colliding with other agents that
move along their paths in the token. Similarly, the edge from state $(l,t)$ to
state $(l',t+1)$ is removed from the state space iff $a_i$ moving from
location $l$ to location $l'$ at timestep $t$ results in it colliding with
other agents that move along their paths in the token. Since cost-minimal paths
need to be found only to endpoints, the path costs from all locations to all
endpoints are computed in a preprocessing phase and then used as h-values for
all A* searches. TP works as follows: The system initializes the token with
the trivial paths where all agents rest in their initial locations [Line
  2]. In each timestep, the system adds all new tasks, if any, to the task set
[Line 4]. Any agent that has reached the end of its path in the token requests
the token once per timestep. (It turns out that one can easily drop the
condition and let any free agent request the token once per timestep for both
decoupled MAPD algorithms in this paper.) The system then sends the token to
each agent that requests it, one after the other [Lines 5-6]. The agent with
the token chooses a task from the task set such that no path of other agents
in the token ends in the pickup or delivery location of the task [Line 7].

\begin{itemize}

\item If there is at least one such task, then the agent assigns itself to the
  one with the smallest h-value from its current location to the pickup
  location of the task and removes this task from the task set [Lines
    9-11]. The agent then calls function Path1 to update its path in the token
  with a cost-minimal path that a) moves from its current location via the
  pickup location of the task to the delivery location of the task and b) does
  not collide with the paths of other agents stored in the token [Line 12].

\item If there is no such task, then the agent does not assign itself to a
  task in the current timestep. If the agent is not in the delivery location
  of a task in the task set, then it updates its path in the token with the
  trivial path where it rests in its current location [Line 14]. Otherwise, to
  avoid deadlocks, it calls function Path2 to update its path in the token
  with a cost-minimal path that a) moves from its current location to an
  endpoint such that the delivery locations of all tasks in the task set are
  different from the chosen endpoint and no path of other agents in the token
  ends in the chosen endpoint and b) does not collide with the paths of other
  agents stored in the token [Line 16].

\end{itemize}

Finally, the agent returns the token to the system and moves along its path in
the token [Lines 17-18].

We now prove that the agent is always able to find a path because it finds a
path only when it is at an endpoint and thus has to find only a path from an
endpoint to an endpoint.

\begin{pro}
\label{p1}
Function Path1 returns a path successfully for well-formed MAPD instances.
\end{pro}

\begin{proof}
  We construct a path from the current location $loc(a_i)$ of agent $a_i$
  (which is an endpoint) via the pickup location $s_j$ of task $\tau_j$ to the
  delivery location $g_j$ of task $\tau_j$ that does not collide with the
  paths of other agents stored in the token. Due to
  Definition~\ref{defn:well-formed}, there exists a path from $loc(a_i)$ via
  $s_j$ to $g_j$ that traverses no other endpoints. All paths stored in
  \emph{token} end in endpoints that are different from $loc(a_i)$, $s_j$, and
  $g_j$. Thus, this path does not collide with the paths of the other agents
  if $a_i$ moves along it after all other agents have moved along their paths.
\end{proof}

\begin{pro}
\label{p2}
Function Path2 returns a path successfully for well-formed MAPD instances.
\end{pro}

\begin{proof}
  Due to Definition~\ref{defn:well-formed}, there exist at least $m$ non-task
  endpoints and thus at least one non-task endpoint such that no path of
  agents other than agent $a_i$ in the token ends in the non-task endpoint. Of
  course, the delivery locations of all tasks in the task set are different
  from the non-task endpoint as well. We construct a path from the current
  location $loc(a_i)$ of agent $a_i$ (which is an endpoint) to the chosen
  endpoint that does not collide with the paths of other agents stored in the
  token. Due to Definition~\ref{defn:well-formed}, there exists a path from
  $loc(a_i)$ to the chosen endpoint that traverses no other endpoints. All
  paths stored in \emph{token} end in endpoints that are different from
  $loc(a_i)$ and the chosen endpoint. Thus, this path does not collide with
  the paths of the other agents if $a_i$ moves along it after all other agents
  have moved along their paths.
\end{proof}

\begin{thm} \label{thm1}
All well-formed MAPD instances are solvable, and TP solves them.
\end{thm}

\begin{proof}
  We show that each task is eventually assigned some agent and executed by
  it. Each agent requests the token after a bounded number of timesteps, and
  no agent rests in the delivery location of a task in the task set due to
  Line 16. Thus, the condition on Line 8 becomes eventually satisfied and some
  agent assigns itself to some task on Line 10. The agent is then able to
  execute it due to Properties \ref{p1} and \ref{p2}.
\end{proof}

\subsection{Token Passing with Task Swaps (TPTS)}

TP is simple but can be made more effective. Token Passing with Task Swaps
(TPTS) is similar to TP except that its task set now contains all unexecuted
tasks, rather than only all tasks that have no agents assigned. This means
that an agent with the token can assign itself not only to a task that has no
agent assigned but also to a task that is already assigned another agent
as long as that agent is still moving to the pickup location of the
task. This might be beneficial when the former agent can move to the pickup
location of the task in fewer timesteps than the latter agent. The latter
agent is then no longer assigned to the task and no longer needs to execute
it. The former agent therefore sends the token to the latter agent so that the
latter agent can try to assign itself to a new task.

\begin{algorithm}[t]
\scriptsize
\renewcommand\arraystretch{0.5}
\caption{Token Passing with Task Swaps (TPTS)}
\label{alg:TPTS}
/* system executes now */\;
Initialize \emph{token} with the (trivial) path $[loc(a_i)]$ for each agent $a_i$\;
\While{true}
{
    Add all new tasks, if any, to the task set $\mathcal{T}$\;
    \While{\textnormal{agent $a_i$ exists that requests \emph{token}}}
    {
      /* system sends \emph{token} to $a_i$ - $a_i$ executes now */\;
      GetTask($a_i$, \emph{token})\;
      /* $a_i$ returns \emph{token} to system - system executes now */\;
    }
    All agents move along their paths in \emph{token} for one timestep and
    remove tasks from $\mathcal{T}$ when they start to execute them\;
    /* system advances to the next timestep */\;
}
\Fn{\textnormal{GetTask($a_i$, \emph{token})}}
{
  $\mathcal{T'} \gets \{\tau_j \in \mathcal{T}| $no other path in \emph{token} ends in $s_j$ or $g_j$$\}$\;
    \While{$\mathcal{T'}\neq \emptyset$}
    {
      $\tau \gets \arg \min_{\tau_j \in \mathcal{T'}} h(loc(a_i), s_j)$\;
      Remove $\tau$ from $\mathcal{T'}$\;
      \If{\textnormal{no agent is assigned to $\tau$}}
        {
          Assign $a_i$ to $\tau$\;
          Update $a_i$'s path in \emph{token} with Path1($a_i$, $\tau$, \emph{token})\;
          \textbf{return} \emph{true}\;
        }
        \Else
        {
          Remember \emph{token}, task set, and agent assignments\;
          $a_{i'} \gets$ agent that is assigned to $\tau$\;
          Unassign $a_{i'}$ from $\tau$ and assign $a_i$ to $\tau$\;
          Remove $a_{i'}$'s path from \emph{token}\;
          Path1($a_i$, $\tau$, \emph{token})\;
          Compare when $a_i$ reaches $s_j$ on its path in \emph{token} to when
          $a_{i'}$ reaches $s_j$ on its path in \emph{token}'\;
          \If{$a_i$ \textnormal{reaches} $s_j$ \textnormal{earlier than} $a_{i'}$}
             {
               /* $a_i$ sends \emph{token} to $a_{i'}$ - $a_{i'}$ executes now */\;
               \emph{success} $\gets$ GetTask($a_i'$, \emph{token})\;
               /* $a_{i'}$ returns \emph{token} to $a_i$ - $a_i$ executes now */\;
               \If{success}
                  {
                    \textbf{return} \emph{true}\;
                  }
             }
             Restore \emph{token}, task set, and agent assignments\;
        }
    }
    \If{$loc(a_i)$ \textnormal{is not an endpoint}} {
        Update $a_i$'s path in \emph{token} with Path2($a_i$, \emph{token})\;
        \If{\textnormal{path was found}} {
          \textbf{return} \emph{true}\;
        }
    }
    \Else {
      \If{no task $\tau_j \in \mathcal{T}$ exists with $g_j = loc(a_i)$}
         {
           Update $a_i$'s path in \emph{token} with the path $[loc(a_i)]$\;
         }
         \Else
             {
               Update $a_i$'s path in \emph{token} with Path2($a_i$, \emph{token})\;
             }
             \textbf{return} \emph{true}\;
    }
    \textbf{return} \emph{false}\;
}
\end{algorithm}

Algorithm~\ref{alg:TPTS} shows the pseudo-code of TPTS. It uses the same main
loop [Lines 3-10] and the same functions Path1 and Path2 as TP. Agent $a_i$
with the token executes function GetTask [Line 7], where it tries to assign
itself to a task in the task set $\mathcal{T}$ and find a path to an endpoint. The
call of function GetTask returns success (true) if agent $a_i$ finds a path to
an endpoint and failure (false) otherwise.

When executing function GetTask, agent $a_i$ considers all tasks $\tau$ from
the task set such that no path of other agents in the token ends in the pickup
or delivery location of the task [Line 12], one after the other in order of
increasing h-values from its current location to the pickup locations of the
tasks [Lines 13-15]. If no agent is assigned to the task, then (as in TP)
agent $a_i$ assigns itself to the task, updates its path in the token with
function Path1 and returns success [Lines 17-19].  Otherwise, agent $a_i$
unassigns the agent $a_{i'}$ assigned to the task and assigns itself to the
task [Line 23]. It removes the path of agent $a_{i'}$ from the token and
updates its own path in the token with function Path1 [Lines 24-25]. If agent
$a_i$ reaches the pickup location of the task with fewer timesteps than agent
$a_{i'}$, then it sends the token to agent $a_{i'}$, which executes function
GetTask to try to assign itself to a new task and eventually returns the token
to agent $a_i$ [Lines 28-30]. If agent $a_{i'}$ returns success, then agent
$a_i$ returns success as well [Lines 31-32]. In all other cases, agent $a_i$
reverses all changes to the paths in the token, task set, and agent
assignments and then considers the next task $\tau$ [Lines 33].

Once agent $a_i$ has considered all tasks $\tau$ unsuccessfully, then it does
not assign itself to a task in the current timestep. If it is not in an
endpoint (which can happen only during a call of Function GetTask on Line 29),
then it updates its path in the token with function Path2 to move to an
endpoint [Line 35]. The call can fail since agent $a_i$ is not at an endpoint.
Agent $a_i$ returns success or failure depending on whether it was able to
find a path [Lines 37 and 44]. Otherwise, (as in TP) if the agent is not in
the delivery location of a task in the task set, then it updates its path in
the token with the trivial path where it rests in its current location [Line
  40]. Otherwise, to avoid deadlocks, it updates its path in the token with
Function Path2 [Line 42]. In both cases, it returns success [Line 43].

Finally, the agent returns the token to the system and moves along its path in
the token, removing the task that it is assigned to (if any) from the task set
once it reaches the pickup location of the task (and thus starts to execute
it) [Lines 8-9].

\begin{pro}\label{p3}
  Function GetTask returns successfully for well-formed MAPD instances when
  called on Line 7.
\end{pro}

\begin{proof}
  Function GetTask returns in finite time (and the number of times an agent
  can unassign another agent from any task is bounded during its execution)
  because a) the number of tasks in the task set is finite; b) an agent can
  unassign another agent from a task only if it reaches the pickup location of
  the task with fewer timesteps than the other agent; c) a task that has some
  agent assigned always continues to have some agent assigned until it has
  been executed; and d) the functions Path1 and Path2 return in finite
  time. We show that an agent that executes function GetTask on Line 7 finds a path to an
  endpoint. The agent is always able to find paths with functions Path1 and
  Path 2 on all lines but Line 35 because it is then at an endpoint and thus
  has to find a path from an endpoint to an endpoint. The proofs are similar
  to those of Properties \ref{p1} and \ref{p2}. However, the agent is not
  guaranteed to find a path with function Path2 on Line 35 because it is then
  not at an endpoint and thus has to find a path from a non-endpoint to an
  endpoint.  Since the agent is at an endpoint during the call of function
  GetTask on Line 7, it does not execute Line 35, finds a path, and returns
  success.
\end{proof}

\begin{thm}\label{thm2}
TPTS solves all well-formed MAPD instances.
\end{thm}

\begin{proof}
  The proof is similar to the one of Theorem \ref{thm1} but uses Property
  \ref{p3}.
\end{proof}

\begin{figure}
  \centering
  \includegraphics[width=0.2\columnwidth]{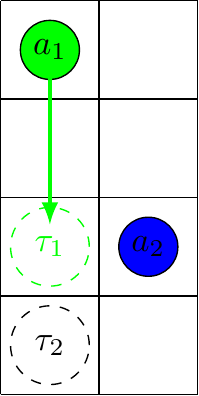}~~
  \includegraphics[width=0.2\columnwidth]{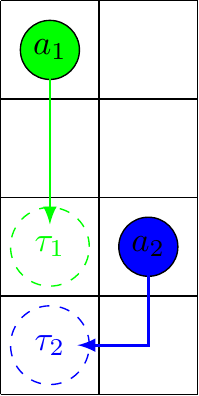}~~
  \includegraphics[width=0.2\columnwidth]{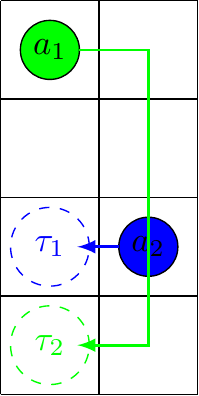}
  \caption{The figure shows a MAPF instance. The pickup location is the same
    as the delivery location for both two tasks. Blue and green circles
    are the initial locations of agents. Dashed circles are the
    pickup/delivery locations.}
  \label{fig:TPTS}
\end{figure}

TPTS is often more effective than TP but Figure~\ref{fig:TPTS} shows that this
is not guaranteed. The figure shows a MAPD instance with two agents $a_1$ and
$a_2$ and two tasks $\tau_1$ and $\tau_2$. The pickup location is the same as
the delivery location for each task. Assume that both $a_1$ and $a_2$ request
the token and the system sends it to $a_1$ first. $a_1$ assigns itself to
$\tau_1$. Figure~\ref{fig:TPTS} (left) shows the path of $a_1$. The system
then sends the token to $a_2$ next. In TP, $a_2$ assigns itself to $\tau_2$.
Figure \ref{fig:TPTS} (center) shows the paths of $a_1$ and $a_2$. The
resulting service time is two. In TPTS, however, $a_2$ assigns itself to
$\tau_1$ because it can reach the pickup location of $\tau_1$ with fewer
timesteps than $a_1$. In return, $a_1$ assigns itself to $\tau_2$. Figure
\ref{fig:TPTS} (right) shows the paths of $a_1$ and $a_2$. The service time is
three (the average of five and one).

\section{Centralized Algorithm}

In this section, we develop a centralized strawman MAPD algorithm, CENTRAL, to
evaluate how effective our decoupled MAPD algorithms are. We want CENTRAL to
be reasonably efficient and effective but do not require that it is optimally
effective or even solves all well-formed MAPD instances. The agents of a
centralized MAPD algorithm can communicate more than the ones of TPTS, and the
ones of TPTS communicate more than the ones of TP. Thus, we expect the MAPD
algorithms to be in increasing order of effectiveness: TP, TPTS, and
CENTRAL. For the same reason, we expect the MAPD algorithms to be in
increasing order of efficiency: CENTRAL, TPTS, and TP.

In each timestep, CENTRAL first assigns endpoints to all agents and then
solves the resulting MAPF instance to plan paths for all agents from their
current locations to their assigned endpoints simultaneously. Finally, all
agents move along their paths for one timestep and the procedure repeats.

\noindent \textbf{Agent Assignment} First, CENTRAL considers each agent, one
after the other, that rests at the pickup location of an unexecuted task. If the
delivery location of the task is currently not assigned to other agents,
CENTRAL assigns the agent to the corresponding unexecuted task (if it is not
assigned to the task already) and assigns the delivery location of the task to
the agent. The agent then starts to execute the task and thus becomes
occupied. Then, CENTRAL assigns each free agent either the pickup location of
an unexecuted task or some other endpoint as parking location. To make the
resulting MAPF problem solvable, the endpoints assigned to all agents must be
pairwise different. Agents are assigned pickup locations of unexecuted tasks
in order to execute the tasks afterward. Thus, when CENTRAL assigns pickup
locations of unexecuted tasks to agents, we want the delivery locations of
these tasks to be different from the endpoints assigned to all agents (except
for their own pickup locations) and from each other. CENTRAL achieves these
constraints as follows:

First, CENTRAL greedily constructs a set of possible endpoints $X$ for the
free agents as follows: CENTRAL greedily constructs a subset ${\mathcal T}'$
of unexecuted tasks, starting with the empty set, by checking for each
unexecuted task, one after the other, whether its pickup and delivery
locations are different from the delivery locations of all executed tasks and
the pickup and delivery locations of all unexecuted tasks already added to
${\mathcal T}'$ and, if so, adds it to ${\mathcal T}'$. CENTRAL then sets $X$
to the pickup locations of all tasks in ${\mathcal T}'$. If the number of free
agents is larger than $|X|$, then CENTRAL needs to add endpoints to $X$ as
parking locations for some free agents.  Since it is not known a priori which
free agents these parking locations will be assigned to, there should be one good
parking location available for each free agent, which is possible due to
Definition~\ref{defn:well-formed}. CENTRAL thus greedily determines a good
parking location for each free agent $a_i$, one after the other, as the endpoint
$e$ that minimizes the cost $c(a_i,e)$ (``is closest to the agent'') among all
endpoints that are different from the delivery locations of all executed
tasks, the pickup and delivery locations of all tasks in ${\mathcal T}'$, and
the parking locations already determined, where $c(a_i,e)$ is the cost of a
cost-minimal path that moves from the current location of free agent $a_i$ to
endpoint $e$. It then adds this endpoint to $X$.

Second, CENTRAL assigns each free agent an endpoint in $X$ to satisfy all
constraints. It uses the Hungarian Method \cite{Kuhn1955} for this purpose
with the modified costs $c'(a_i,e)$ for each pair of free agent $a_i$ and
endpoint $e$, where $c$ is the number of free agents, $C$ is a sufficiently
large constant (for example, the maximum over all costs $c(a_i,e)$ plus one),
and $c'(a_i,e) = c \cdot C \cdot c(a_i,e)$ if $e$ is a pickup location of a
task in ${\mathcal T}'$ and $c'(a_i,e) = c \cdot C^2 + c(a_i,e)$ if $e$ is a
parking location. The modified costs have two desirable properties: a) The
modified cost of assigning a pickup location to a free agent is always smaller
than the modified cost of assigning a parking location to the same agent.
Therefore, assigning pickup locations is more important than assigning rest
locations. b) Assigning a closer pickup location to a single free agent that
is assigned a pickup location reduces the total modified cost more than
assigning closer parking locations to all free agents that are assigned rest
locations.  Therefore, assigning closer pickup locations is more important
than assigning closer parking locations.

\noindent \textbf{Path Planning} CENTRAL uses the optimally effective MAPF
algorithm Conflict-Based Search \cite{DBLP:journals/ai/SharonSFS15} to plan
collision-free paths for all agents from their current locations to their
assigned endpoints simultaneously. These paths minimize the sum of the number
of timesteps required by all agents to reach their assigned endpoints and stop
moving. We noticed that CENTRAL becomes significantly more efficient if it
plans paths in two stages: \hang{First, it plans paths for all agents that
  become occupied in the current timestep to their assigned endpoints (using
  the approach described above but treating the most recently calculated paths
  of all other agents as spatio-temporal obstacles.}
Then, it plans paths for all free agents to their assigned endpoints (again
using the approach described above but treating the most recently calculated
paths of all other agents as spatio-temporal obstacles). In general, two
smaller MAPF instances can be solved much faster than their union due to the
NP-hardness of the problem. Also, CENTRAL can then determine a more informed
cost $c(a_i,e)$ as the cost of a cost-minimal path that a) moves from the
current location of agent $a_i$ to endpoint $e$ and b) does not collide with
the paths of the occupied agents (as described for TP).


\hang{
\begin{pro}
Path planning for all agents that became occupied in the current timestep
returns paths successfully for well-formed MAPD instances.
\end{pro}
}

\begin{proof}
\hang{We construct paths for all agents that became occupied in the current
  timestep} from their current locations to their assigned endpoints that do
not collide with the most recently calculated paths of all other agents:
Assume that all other agents move along their most recently calculated
paths. When all of them have reached the ends of their paths, move all agents
that became occupied one after the other to their assigned endpoints, which is
possible due to Definition~\ref{defn:well-formed} since their current
locations are endpoints and their assigned endpoints are different from the
endpoints that all other agents now occupy.
\end{proof}

\begin{pro}
Path planning for all free agents returns paths successfully for well-formed
MAPD instances.
\end{pro}

\begin{proof}
We construct paths for all free agents from their current locations to their
assigned endpoints that do not collide with the most recently calculated paths
of all other agents: Assume that all agents move along their most recently
calculated paths. When all of them have reached the ends of their paths, move
all free agents one after the other to their assigned endpoints, which is
possible due to Definition~\ref{defn:well-formed} since the locations that they
now occupy are endpoints. Directly before an agent moves to its assigned
endpoint, check whether this endpoint is blocked by another agent. If so, move
this other agent to an unoccupied endpoint first. Such an endpoint exists
since there are at least $m+1$ endpoints for $m$ agents due to
Definition~\ref{defn:well-formed}.
\end{proof}

\section{Experimental Evaluation}

\begin{figure}[t]
  \centering
  \includegraphics[width=0.75\columnwidth]{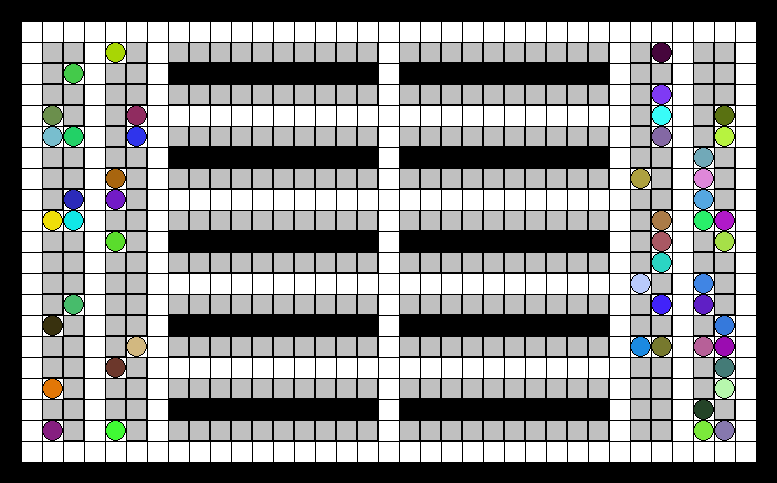}
  \caption{The figure shows a $21\times 35$ 4-neighbor grid that represents
    the layout of a small simulated warehouse environment with 50 agents. Black
    cells are blocked. Gray cells are task endpoints. Colored circles are the
    initial locations of agents.}
  \label{fig:map0}
\end{figure}

In this section, we describe our experimental results on a 2.50 GHz Intel Core
i5-2450M laptop with 6 GB RAM. We ran TP, TPTS, and CENTRAL in the small
simulated warehouse environment shown in Figure~\ref{fig:map0}. We generated a
sequence of 500 delivery tasks by randomly choosing their pickup and delivery
locations from all task endpoints. The initial locations of the agents are the
only non-task endpoints. We used 6 different task frequencies (numbers of
tasks that are added (in order) from the sequence to the task set in each
timestep): 0.2 (one task every 5 timesteps), 0.5, 1, 2, 5, and 10. For each
task frequency, we used 5 different numbers of agents: 10, 20, 30, 40, and 50.
Table~\ref{tab:1} reports the makespans, the service times, the runtimes per
timestep (in ms), the ratios of the service times of TPTS and TP, and the
ratios of the service times of CENTRAL and TP. The measures for a task
frequency of 10 tasks per timestep are reasonably representative of the case
where all tasks are added in the beginning of the operation since the tasks
are added over the first 50 timesteps only.

\begin{table}[t]
\Huge
\centering
\caption{The table shows the experimental results for TP, TPTS, and CENTRAL
  in the small simulated warehouse environment.}
\label{tab:1}
\resizebox{\columnwidth}{!}{%
\begin{tabular}{cc|rrr|rrrr|rrrr}\hline
                                                               & \multicolumn{1}{c|}{}     & \multicolumn{3}{c|}{TP}                                                                                                                                                                 & \multicolumn{4}{c|}{TPTS}                                                                                                                                                                                                                                              & \multicolumn{4}{c}{CENTRAL}                                                                                                                                                                                                                                           \\
                                                               \hline
\begin{tabular}[c]{@{}c@{}}task fre-\\quency\end{tabular} & \multicolumn{1}{c|}{\begin{tabular}[c]{@{}c@{}}agents\end{tabular}} & \multicolumn{1}{c}{\begin{tabular}[c]{@{}c@{}}make-\\span\end{tabular}} & \multicolumn{1}{c}{\begin{tabular}[c]{@{}c@{}}service\\time\end{tabular}} & \multicolumn{1}{c|}{\begin{tabular}[c]{@{}c@{}}run-\\time\end{tabular}} & \multicolumn{1}{c}{\begin{tabular}[c]{@{}c@{}}make-\\span\end{tabular}} & \multicolumn{1}{c}{\begin{tabular}[c]{@{}c@{}}service\\time\end{tabular}} & \multicolumn{1}{c}{\begin{tabular}[c]{@{}c@{}}run-\\time\end{tabular}} & \multicolumn{1}{c|}{ratio} & \multicolumn{1}{c}{\begin{tabular}[c]{@{}c@{}}make-\\span\end{tabular}} & \multicolumn{1}{c}{\begin{tabular}[c]{@{}c@{}}service\\time\end{tabular}} & \multicolumn{1}{c}{\begin{tabular}[c]{@{}c@{}}run-\\time\end{tabular}} & \multicolumn{1}{c}{ratio} \\
\hline
\multirow{5}{*}{0.2}                                           & 10                       & 2,532                      & 38.54                                                                            & 0.13                                                                    & 2,532                      & 29.33                                                                            & 1.86                                                                    & 0.76                                                                         & 2,513                      & 27.78                                                                            & 92.69                                                                   & 0.72                                                                             \\
                                                               & 20                       & 2,540                      & 39.77                                                                            & 0.26                                                                    & 2,520                      & 25.36                                                                            & 9.82                                                                    & 0.64                                                                         & 2,513                      & 24.37                                                                            & 493.83                                                                  & 0.61                                                                             \\
                                                               & 30                       & 2,546                      & 38.71                                                                            & 0.25                                                                    & 2,527                      & 23.88                                                                            & 21.57                                                                   & 0.62                                                                         & 2,513                      & 23.10                                                                            & 1,225.62                                                                 & 0.60                                                                             \\
                                                               & 40                       & 2,540                      & 38.88                                                                            & 0.24                                                                    & 2,524                      & 23.50                                                                            & 27.49                                                                   & 0.60                                                                         & 2,511                      & 22.48                                                                            & 2,246.66                                                                 & 0.58                                                                             \\
                                                               & 50                       & 2,540                      & 40.03                                                                            & 0.32                                                                    & 2,524                      & 23.11                                                                            & 47.33                                                                   & 0.58                                                                         & 2,511                      & 21.82                                                                            & 3,426.01                                                                 & 0.54                                                                             \\
                                                               \hline
\multirow{5}{*}{0.5}                                           & 10                       & 1,309                      & 132.79                                                                           & 0.24                                                                    & 1,274                      & 131.15                                                                           & 0.33                                                                    & 0.99                                                                         & 1,242                      & 116.37                                                                           & 113.60                                                                  & 0.88                                                                             \\
                                                               & 20                       & 1,094                      & 42.69                                                                            & 1.16                                                                    & 1,038                      & 30.74                                                                            & 12.39                                                                   & 0.72                                                                         & 1,031                      & 28.05                                                                            & 309.01                                                                  & 0.66                                                                             \\
                                                               & 30                       & 1,069                      & 43.97                                                                            & 1.51                                                                    & 1,035                      & 27.14                                                                            & 34.04                                                                   & 0.62                                                                         & 1,034                      & 25.36                                                                            & 916.37                                                                  & 0.58                                                                             \\
                                                               & 40                       & 1,090                      & 43.01                                                                            & 0.70                                                                    & 1,038                      & 25.98                                                                            & 19.85                                                                   & 0.60                                                                         & 1,034                      & 24.26                                                                            & 1,697.99                                                                 & 0.56                                                                             \\
                                                               & 50                       & 1,083                      & 43.66                                                                            & 1.36                                                                    & 1,036                      & 25.22                                                                            & 71.48                                                                   & 0.58                                                                         & 1,031                      & 23.83                                                                            & 2,920.72                                                                 & 0.55                                                                             \\
                                                               \hline
\multirow{5}{*}{1}                                             & 10                       & 1,198                      & 311.78                                                                           & 0.20                                                                    & 1,182                      & 301.03                                                                           & 0.37                                                                    & 0.97                                                                         & 1,143                      & 285.67                                                                           & 141.32                                                                  & 0.92                                                                             \\
                                                               & 20                       & 757                       & 95.98                                                                            & 1.03                                                                    & 706                       & 88.25                                                                            & 2.80                                                                    & 0.92                                                                         & 673                       & 74.79                                                                            & 279.52                                                                  & 0.78                                                                             \\
                                                               & 30                       & 607                       & 53.80                                                                            & 2.81                                                                    & 561                       & 42.84                                                                            & 8.45                                                                    & 0.80                                                                         & 557                       & 30.26                                                                            & 446.38                                                                  & 0.56                                                                             \\
                                                               & 40                       & 624                       & 48.80                                                                            & 2.14                                                                    & 563                       & 31.99                                                                            & 39.54                                                                   & 0.66                                                                         & 556                       & 28.27                                                                            & 1,159.76                                                                 & 0.58                                                                             \\
                                                               & 50                       & 597                       & 49.14                                                                            & 3.76                                                                    & 554                       & 30.27                                                                            & 128.13                                                                  & 0.62                                                                         & 552                       & 26.55                                                                            & 2,197.82                                                                 & 0.54                                                                             \\
                                                               \hline
\multirow{5}{*}{2}                                             & 10                       & 1,167                      & 407.62                                                                           & 0.19                                                                    & 1,168                      & 407.24                                                                           & 0.37                                                                    & 1.00                                                                         & 1,121                      & 386.81                                                                           & 143.30                                                                  & 0.95                                                                             \\
                                                               & 20                       & 683                       & 190.76                                                                           & 0.96                                                                    & 667                       & 181.03                                                                           & 2.38                                                                    & 0.95                                                                         & 628                       & 163.79                                                                           & 406.88                                                                  & 0.86                                                                             \\
                                                               & 30                       & 529                       & 114.39                                                                           & 2.31                                                                    & 496                       & 102.69                                                                           & 8.39                                                                    & 0.90                                                                         & 466                       & 88.45                                                                            & 589.13                                                                  & 0.77                                                                             \\
                                                               & 40                       & 464                       & 95.32                                                                            & 3.43                                                                    & 425                       & 72.59                                                                            & 7.32                                                                    & 0.76                                                                         & 385                       & 58.12                                                                            & 837.21                                                                  & 0.61                                                                             \\
                                                               & 50                       & 432                       & 75.63                                                                            & 6.28                                                                    & 383                       & 58.06                                                                            & 126.47                                                                  & 0.77                                                                         & 320                       & 39.25                                                                            & 1,200.10                                                                 & 0.52                                                                             \\
                                                               \hline
\multirow{5}{*}{5}                                             & 10                       & 1,162                      & 473.78                                                                           & 0.20                                                                    & 1,165                      & 473.18                                                                           & 0.41                                                                    & 1.00                                                                         & 1,105                      & 452.50                                                                           & 126.19                                                                  & 0.96                                                                             \\
                                                               & 20                       & 655                       & 247.08                                                                           & 1.02                                                                    & 645                       & 238.02                                                                           & 1.68                                                                    & 0.96                                                                         & 594                       & 224.70                                                                           & 350.13                                                                  & 0.91                                                                             \\
                                                               & 30                       & 478                       & 170.78                                                                           & 2.22                                                                    & 474                       & 167.66                                                                           & 8.58                                                                    & 0.98                                                                         & 426                       & 147.03                                                                           & 595.04                                                                  & 0.86                                                                             \\
                                                               & 40                       & 418                       & 155.33                                                                           & 4.15                                                                    & 396                       & 131.36                                                                           & 12.31                                                                   & 0.85                                                                         & 334                       & 108.39                                                                           & 864.56                                                                  & 0.70                                                                             \\
                                                               & 50                       & 395                       & 124.59                                                                           & 5.92                                                                    & 343                       & 104.86                                                                           & 59.64                                                                   & 0.84                                                                         & 295                       & 86.22                                                                            & 1,388.30                                                                 & 0.69                                                                             \\
                                                               \hline
\multirow{5}{*}{10}                                            & 10                       & 1,163                      & 495.93                                                                           & 0.22                                                                    & 1,172                      & 505.26                                                                           & 0.40                                                                    & 1.02                                                                         & 1,090                      & 472.56                                                                           & 125.55                                                                  & 0.95                                                                             \\
                                                               & 20                       & 643                       & 275.24                                                                           & 1.09                                                                    & 645                       & 258.36                                                                           & 1.87                                                                    & 0.94                                                                         & 607                       & 248.74                                                                           & 379.53                                                                  & 0.90                                                                             \\
                                                               & 30                       & 526                       & 192.01                                                                           & 1.98                                                                    & 491                       & 198.30                                                                           & 10.82                                                                   & 1.03                                                                         & 414                       & 164.41                                                                           & 593.89                                                                  & 0.86                                                                             \\
                                                               & 40                       & 407                       & 154.63                                                                           & 1.65                                                                    & 389                       & 152.49                                                                           & 12.62                                                                   & 0.99                                                                         & 341                       & 128.29                                                                           & 899.81                                                                  & 0.83                                                                             \\
                                                               & 50                       & 333                       & 131.42                                                                           & 5.62                                                                    & 319                       & 126.96                                                                           & 25.32                                                                   & 0.97                                                                         & 277                       & 105.11                                                                           & 1,376.51                                                                 & 0.80\\
                                                               \hline
\end{tabular}%
}
\end{table}

\noindent\textbf{Makespans and Service Times} The MAPD algorithms in
increasing order of their makespans and service times tend to be: CENTRAL,
TPTS, and TP. For example, the service time of TPTS (and CENTRAL) is up to
about 42 percent (and 48 percent, respectively) smaller than the one
of TP for some experimental runs. The makespans tend to be large for low task
frequencies and small for high task frequencies because the number of tasks is
constant and thus more time steps are needed to add all tasks for low task
frequencies. On the other hand, the service times tend to be small for low
task frequencies and high for high task frequencies because the agents tend to
be able to attend to tasks fast if the number of tasks in the system is
small. The makespans and service times tend to be large for small numbers of
agents and small for large numbers of agents because the agents tend to be
able to attend to tasks fast if the number of agents is large (although
congestion increases). The makespans and service times for a task frequency of
0.2 tasks per timestep are about the same for all numbers of agents because 10
agents already attend to all tasks as fast as the MAPD algorithms allow. The
makespans are similar for all MAPD algorithms and all numbers of agents for
the task frequency of 0.2 tasks per timestep because 10 agents already execute
tasks faster than they are added. The makespans then depend largely on how
fast the agents execute the last few tasks. On the other hand, the makespans
of MAPD algorithms increase substantially when tasks pile up because the
agents execute them more slowly than they are added. This allows us to
estimate the smallest number of agents needed for a lifelong operation as a
function of the task frequency and MAPD algorithm. For example, the makespan
of TPTS increases substantially when the number of agents are reduced from 20
to 10 for a task frequency of 1 task per timestep. Thus, one needs between about 10
and 20 agents for a lifelong operation with TPTS.

\noindent\textbf{Runtimes per Timestep} The MAPF algorithms in increasing
order of their runtimes per timestep tend to be: TP, TPTS, and CENTRAL. For
example, the runtime of TPTS (and CENTRAL) is two orders of magnitude larger
than the runtime of TP (and TPTS, respectively) for some experimental
runs. The runtimes of TP are less than 10 milliseconds, the runtimes of TPTS
are less than 200 milliseconds, and the runtimes of CENTRAL are less than
4,000 milliseconds in all experimental runs. We consider runtimes below one
second to allow for real-time lifelong operation. The runtimes tend not to be
correlated with the task frequencies.  They tend to be small for small numbers
of agents and large for large numbers of agents because all agents need to
perform computations, which are not run in parallel in our experiments.

\begin{figure*}[t!p]
                 \includegraphics[width=0.33\textwidth]{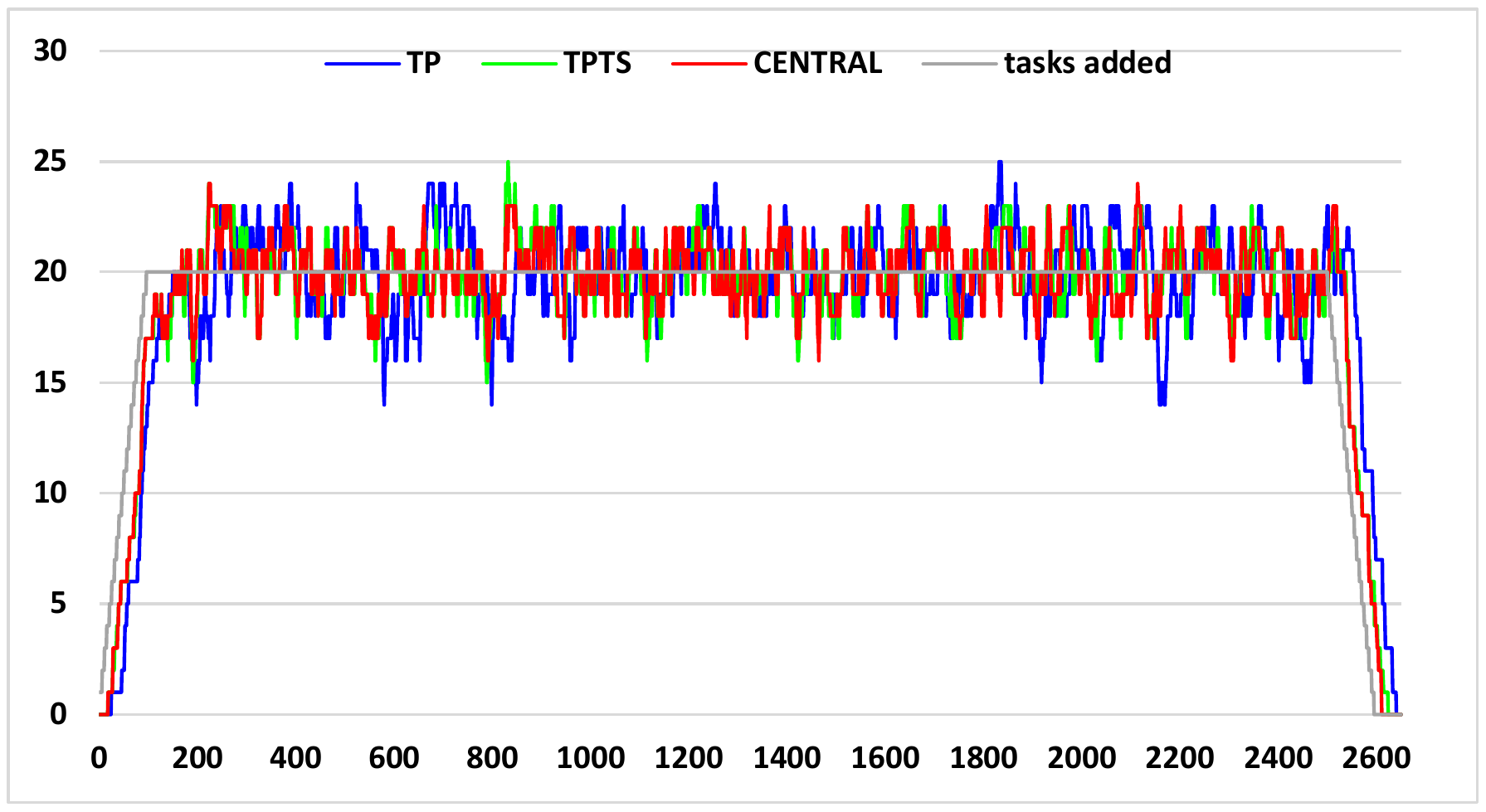}
                 \includegraphics[width=0.33\textwidth]{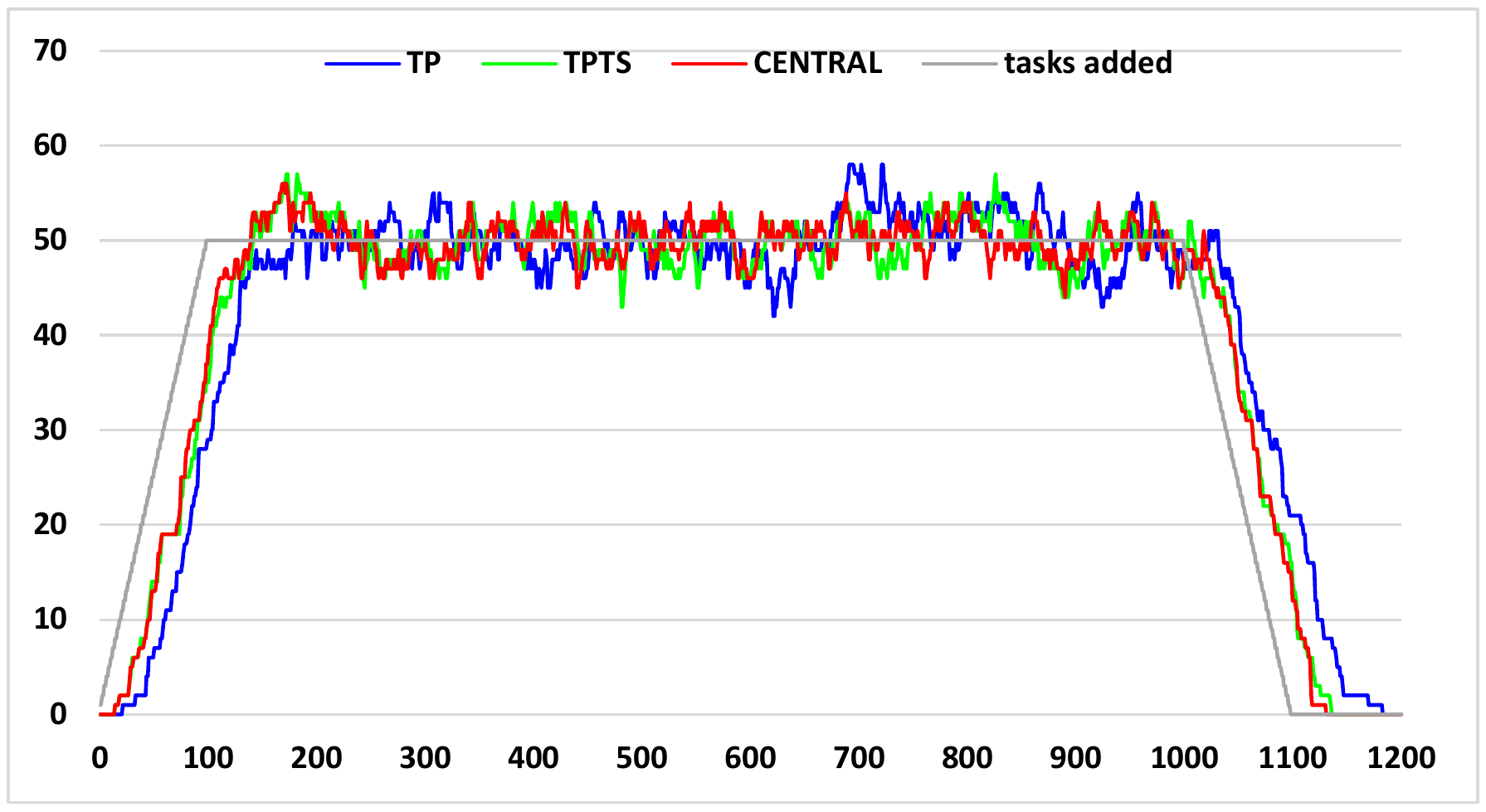}
                 \includegraphics[width=0.33\textwidth]{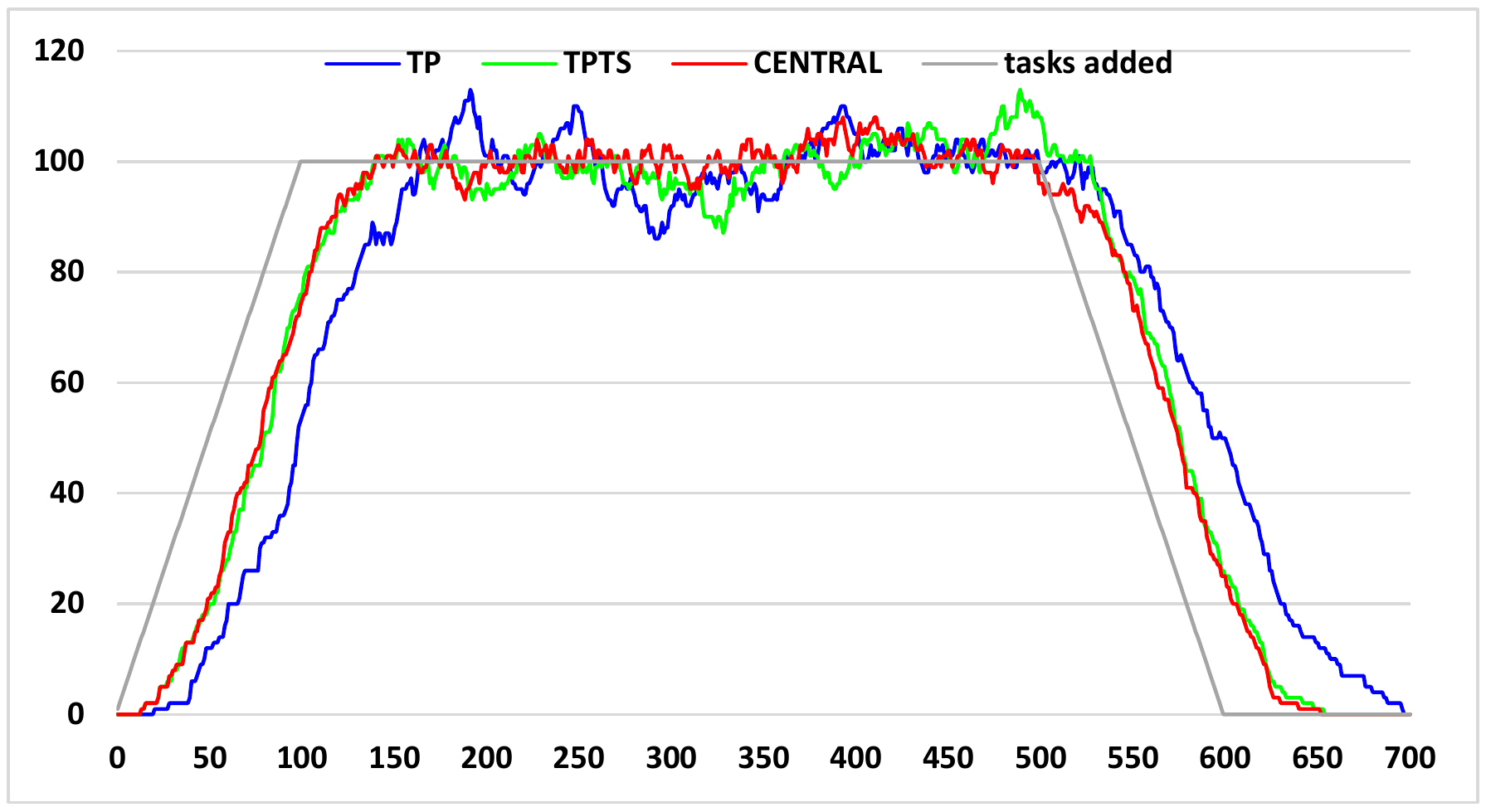}
             \includegraphics[width=0.33\textwidth]{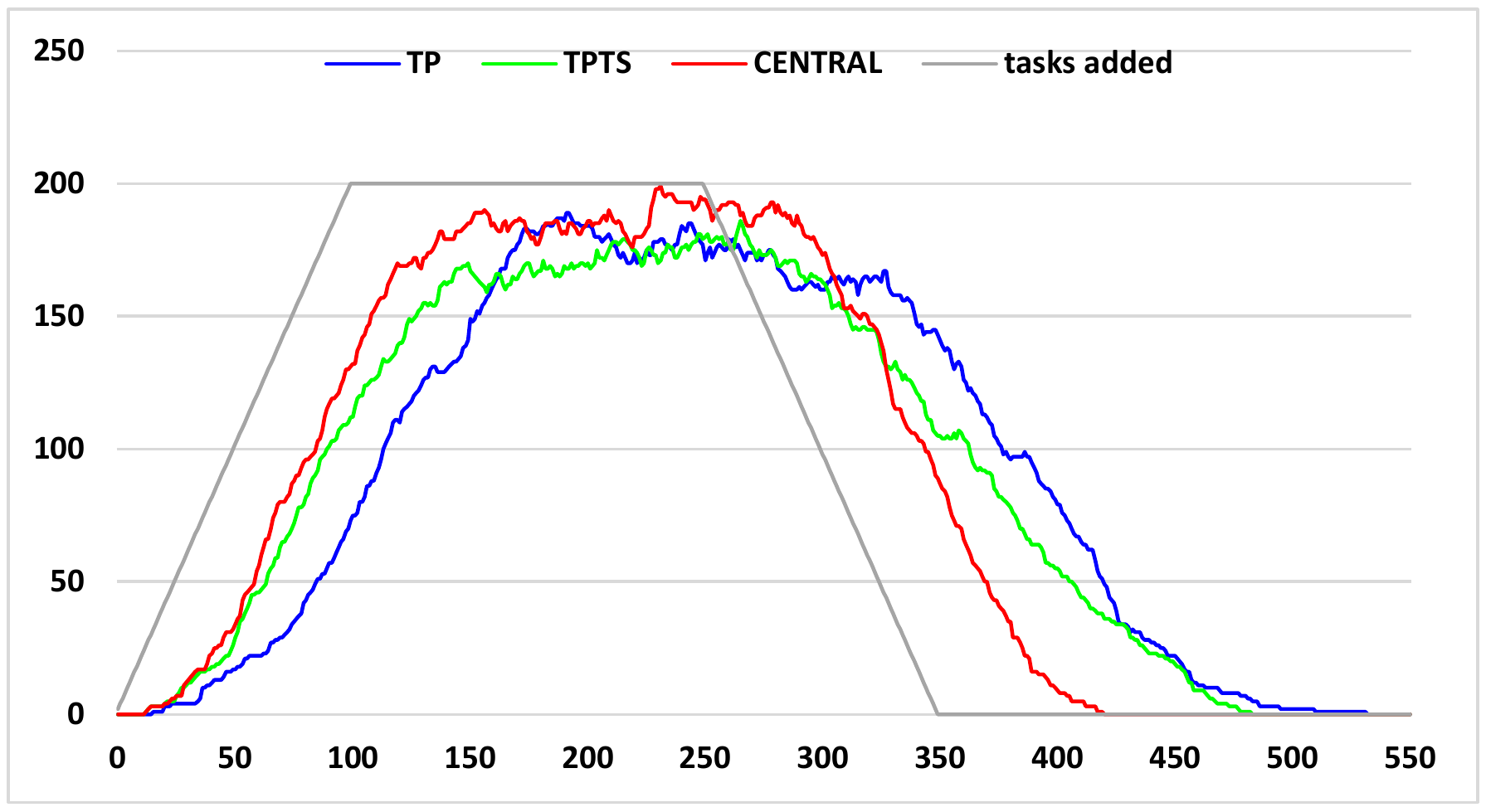}
             \includegraphics[width=0.33\textwidth]{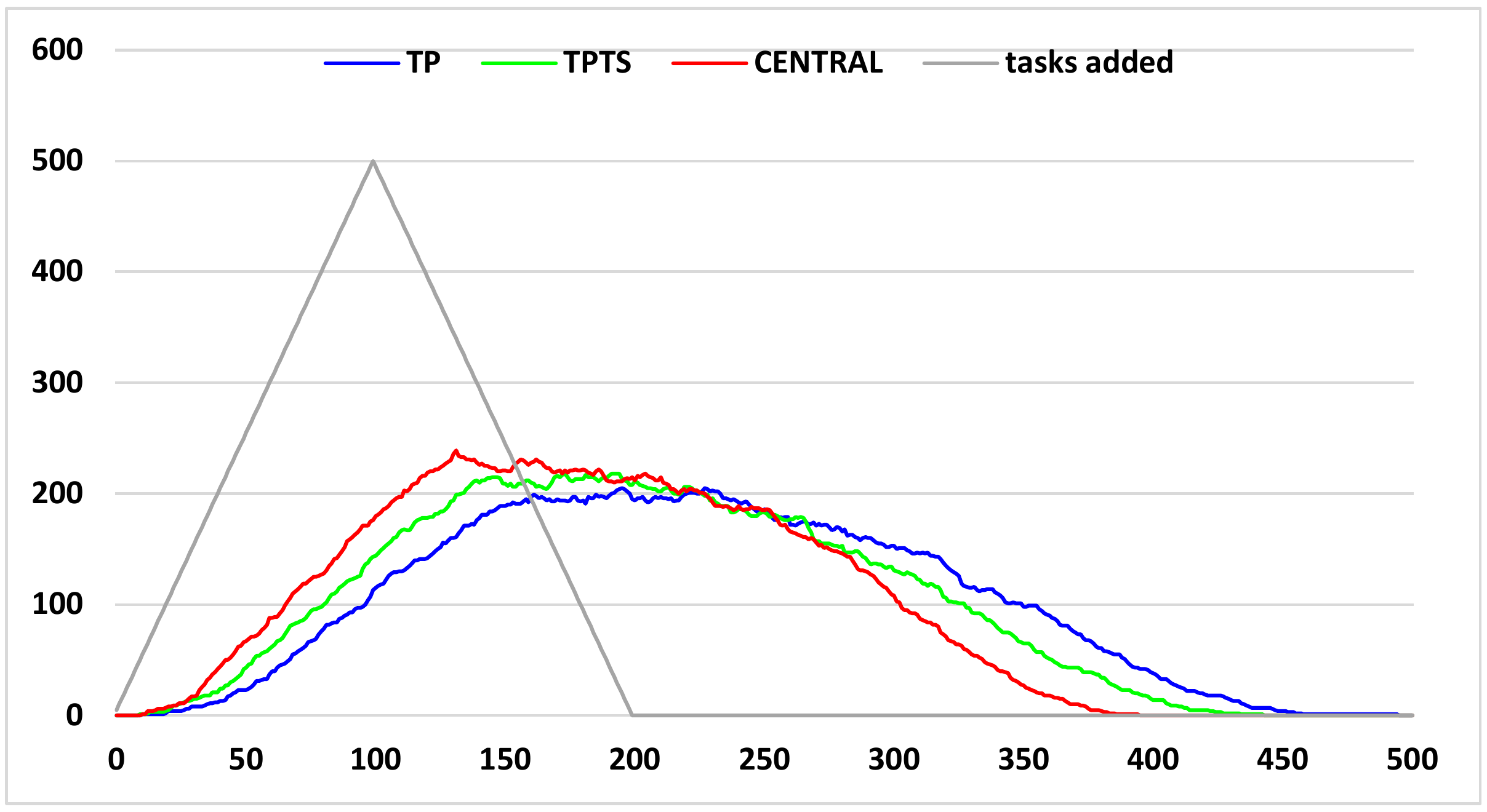}
             \includegraphics[width=0.33\textwidth]{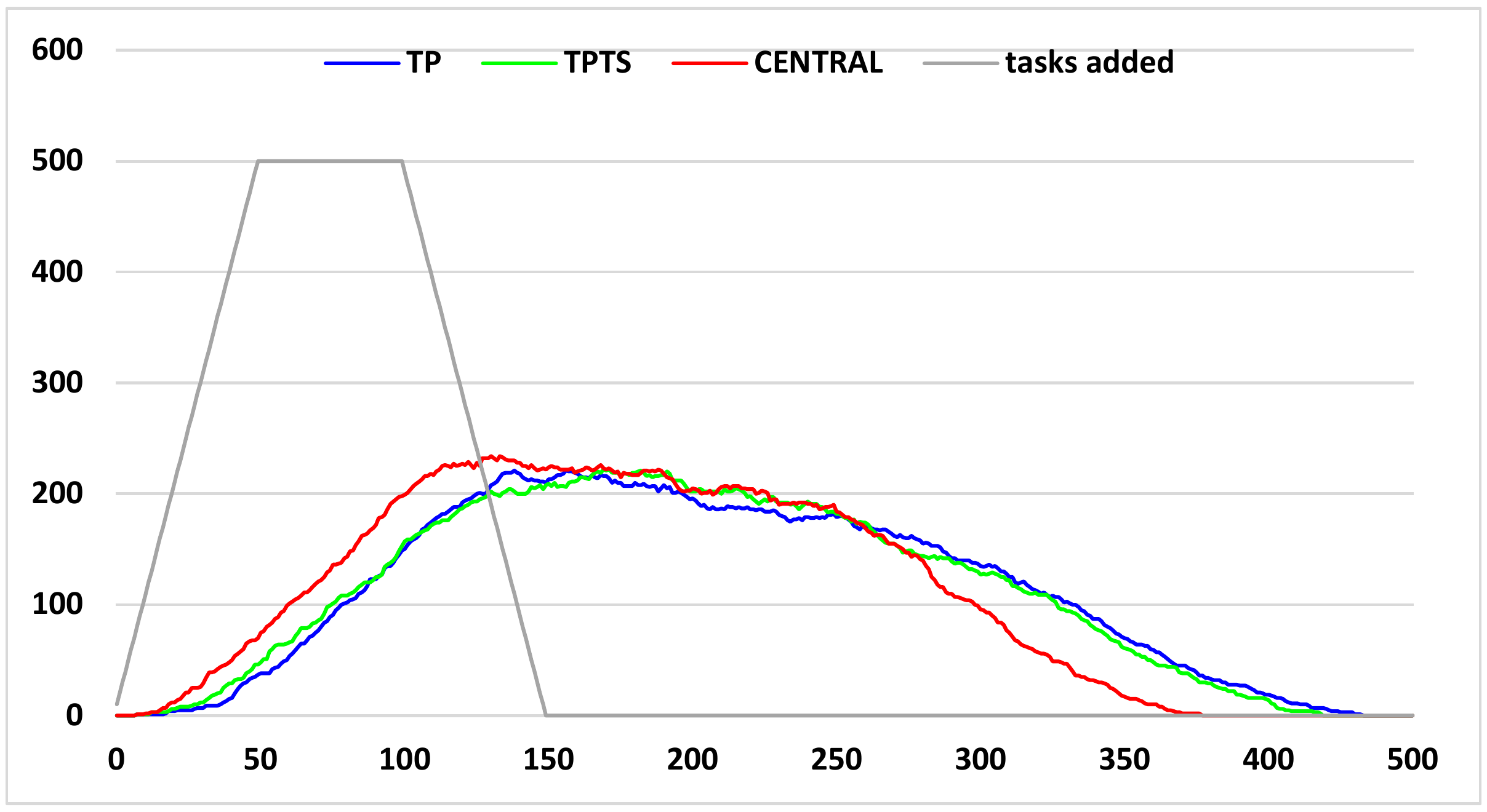}
             \caption{The charts show the number of tasks added (grey) and
               executed by 50 agents during the 100-timestep window $[t-99,t]$
               for TP, TPTS, and CENTRAL as a function of the timestep $t$ for
               task frequencies 0.2, 0.5, 1, 2, 5, and 10 (from left to right
               and top to bottom)} \label{fig:exp}
\end{figure*}

\noindent\textbf{Number of Executed Tasks} The service times vary over time
since only very few tasks are available in the first and last time steps. The
steady state is in between these two extremes. Figure~\ref{fig:exp} therefore
visualizes the number of tasks added and executed by 50 agents during the
100-timestep window $[t-99,t]$ for all MAPD algorithms as a function of the
timestep $t$. For low task frequencies, the numbers of tasks added match the
numbers of tasks executed closely for all MAPD algorithms. Differences between
them arise for higher task frequencies. For example, for the task frequency of
2 tasks per timestep, the number of tasks executed by CENTRAL increases faster
and reaches a higher level than the numbers of tasks executed by TP and
TPTS. The 100-timestep window $[150,249]$ at time step $t=249$ is a close
approximation of the steady state since all tasks are added at a steady rate
during the first 250 timesteps. CENTRAL executes more tasks during this
100-timestep window than TP and TPTS and thus has a smaller service
time. However, the numbers of tasks executed are smaller for all MAPD
algorithms than the number of tasks added, and tasks thus pile up for all of
them in the steady state.

\begin{figure}[t]
  \centering
  \includegraphics[width=0.75\columnwidth]{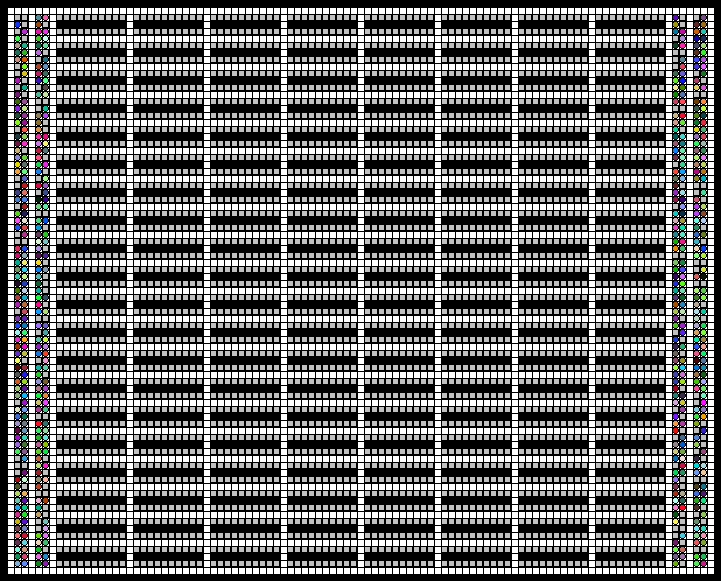}
  \caption{The figure shows a $81\times 81$ 4-neighbor grid that represents
    the layout of a large simulated warehouse environment with 500 agents.}
  \label{fig:map1}
\end{figure}

\begin{figure}[t]
  \centering
  \resizebox{0.75\columnwidth}{!}{%
  \Huge
  \begin{tabular}{c|rrrrr}
    \hline
                 agents & \multicolumn{1}{c}{100} & \multicolumn{1}{c}{200} & \multicolumn{1}{c}{300} & \multicolumn{1}{c}{400} & \multicolumn{1}{c}{500} \\
    \hline
    service time   & 463.25&    330.19&	301.97&	289.08&	284.24 \\
    runtime        & 90.83&	538.22&	1,854.44&	3,881.11&	6,121.06 \\
    \hline
  \end{tabular}
  }
  \includegraphics[width=0.75\columnwidth]{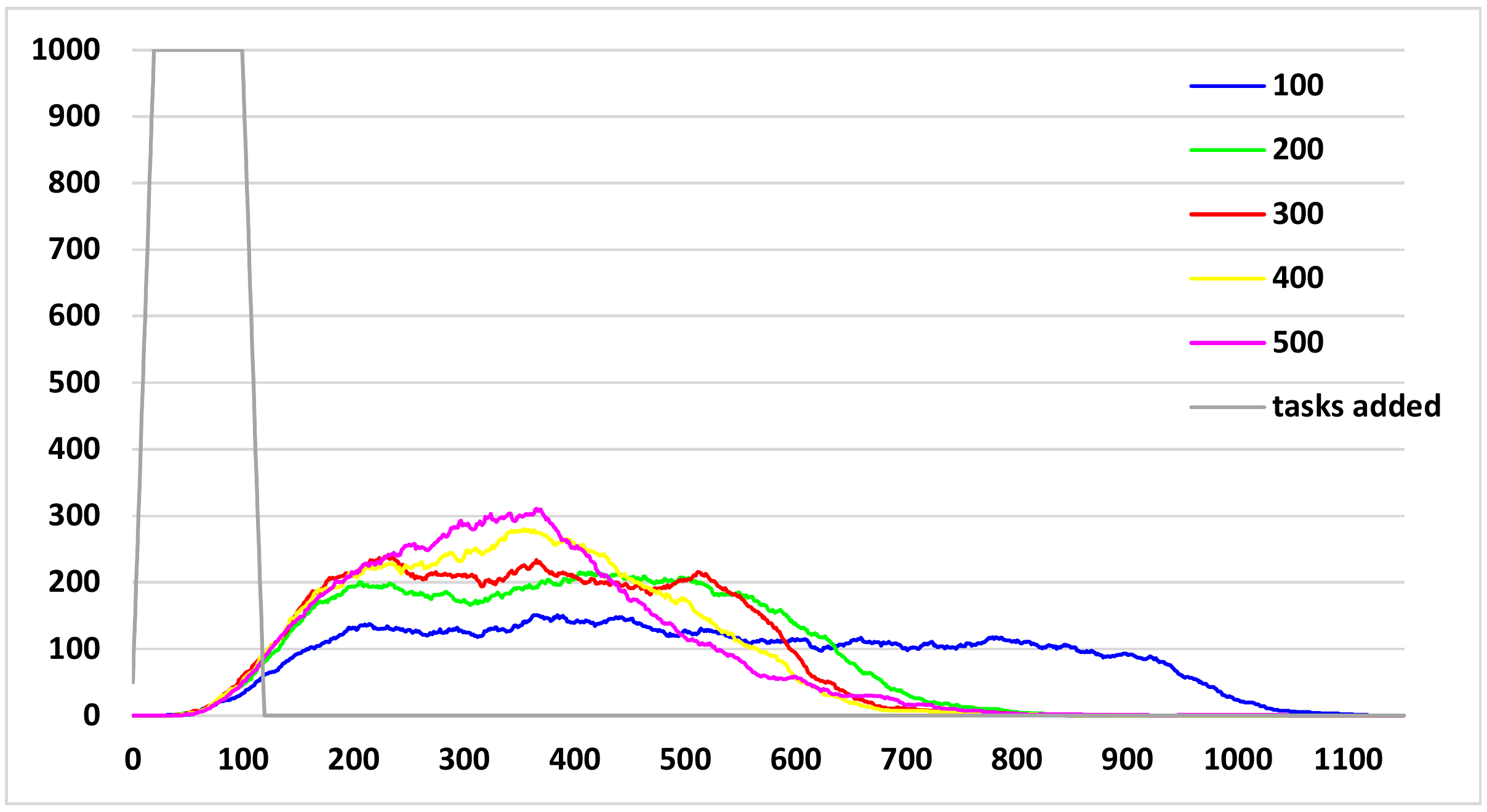}
  \caption{The figure shows the experimental results for TP
  in the large simulated warehouse environment.}
  \label{fig:exp2}
\end{figure}

\noindent\textbf{Scalability} To evaluate how the MAPD algorithms scale in the
size of the environment, we ran TP, TPTS, and CENTRAL in the large simulated
warehouse environment shown in Figure~\ref{fig:map1}. We generated a sequence
of 1,000 delivery tasks by randomly choosing their pickup and delivery
locations from all task endpoints. The initial locations of the agents are the
only non-task endpoints. We used a task frequency of 50 tasks per timestep and
100, 200, 300, 400, and 500 agents. TPTS and CENTRAL did not allow for
real-time lifelong operation for large numbers of agents.
Figure~\ref{fig:exp2} therefore reports, for TP only, the service times and
the runtimes per timestep (in ms) in the table as well as the numbers of tasks
added and executed during a 100-timestep window for different numbers of
agents in the charts.  The runtime of TP is smaller than 500 milliseconds for
200 agents, allowing for real-time lifelong operation.

\section{Conclusions}

In this paper, we studied a lifelong version of the multi-agent path finding
(MAPF) problem, called the multi-agent pickup and delivery (MAPD) problem, to
capture important characteristics of many real-world domains. In the MAPD
problem, agents have to attend to a stream of delivery tasks in an online
setting by first moving to the pickup locations of the tasks and then to the
delivery locations of the tasks while avoiding collisions with other
agents. We presented two decoupled MAPD algorithms, Token Passing (TP) and
Token Passing with Task Swaps (TPTS).  Theoretically, we showed that both MAPD
algorithms solve all well-formed MAPD instances. Experimentally, we compared
them against the centralized strawman MAPD algorithm CENTRAL without this
guarantee in a simulated warehouse system. The MAPD algorithms in increasing
order of their makespans and service times tend to be: CENTRAL, TPTS, and
TP. The MAPF algorithms in increasing order of their runtimes per timestep
tend to be: TP, TPTS, and CENTRAL. TP can easily be extended to a fully
distributed MAPD algorithm and is the best choice when real-time computation
is of primary concern since it remains efficient for MAPD instances with
hundreds of agents and tasks. TPTS requires limited communication among agents
and balances well between TP and CENTRAL.

\small
\bibliographystyle{abbrv}
\bibliography{references}

\begin{thebibliography}{10}

\bibitem{ICBS}
E.~Boyarski, A.~Felner, R.~Stern, G.~Sharon, D.~Tolpin, O.~Betzalel, and S.~E.
  Shimony.
\newblock {ICBS}: Improved conflict-based search algorithm for multi-agent
  pathfinding.
\newblock In {\em International Joint Conference on Artificial Intelligence},
  pages 740--746, 2015.

\bibitem{CapVK15}
M.~C{\'{a}}p, J.~Vokr{\'{\i}}nek, and A.~Kleiner.
\newblock Complete decentralized method for on-line multi-robot trajectory
  planning in well-formed infrastructures.
\newblock In {\em International Conference on Automated Planning and
  Scheduling}, pages 324--332, 2015.

\bibitem{CohenUK16}
L.~Cohen, T.~Uras, T.~K.~S. Kumar, H.~Xu, N.~Ayanian, and S.~Koenig.
\newblock Improved solvers for bounded-suboptimal multi-agent path finding.
\newblock In {\em International Joint Conference on Artificial Intelligence},
  pages 3067--3074, 2016.

\bibitem{PushAndRotate}
B.~de~Wilde, A.~W. ter Mors, and C.~Witteveen.
\newblock Push and rotate: Cooperative multi-agent path planning.
\newblock In {\em International Conference on Autonomous Agents and Multi-Agent
  Systems}, pages 87--94, 2013.

\bibitem{erdem2013general}
E.~Erdem, D.~G. Kisa, U.~Oztok, and P.~Schueller.
\newblock A general formal framework for pathfinding problems with multiple
  agents.
\newblock In {\em {AAAI} Conference on Artificial Intelligence}, pages
  290--296, 2013.

\bibitem{ErdmannL87}
M.~A. Erdmann and T.~Lozano{-}P{\'{e}}rez.
\newblock On multiple moving objects.
\newblock {\em Algorithmica}, 2:477--521, 1987.

\bibitem{EPEJAIR}
M.~Goldenberg, A.~Felner, R.~Stern, G.~Sharon, N.~R. Sturtevant, R.~C. Holte,
  and J.~Schaeffer.
\newblock {Enhanced Partial Expansion A*}.
\newblock {\em Journal of Artificial Intelligence Research}, 50:141--187, 2014.

\bibitem{HoenigICAPS16}
W.~H\"onig, T.~K.~S. Kumar, L.~Cohen, H.~Ma, H.~Xu, N.~Ayanian, and S.~Koenig.
\newblock Multi-agent path finding with kinematic constraints.
\newblock In {\em International Conference on Automated Planning and
  Scheduling}, pages 477--485, 2016.

\bibitem{HoenigIROS16}
W.~H\"onig, T.~K.~S. Kumar, H.~Ma, N.~Ayanian, and S.~Koenig.
\newblock Formation change for robot groups in occluded environments.
\newblock In {\em IEEE/RSJ International Conference on Intelligent Robots and
  Systems}, pages 4836--4842, 2016.

\bibitem{KhorshidHS11}
M.~Khorshid, R.~Holte, and N.~Sturtevant.
\newblock A polynomial-time algorithm for non-optimal multi-agent pathfinding.
\newblock In {\em Annual Symposium on Combinatorial Search}, 2011.

\bibitem{Kuhn1955}
H.~W. Kuhn.
\newblock The {Hungarian} method for the assignment problem.
\newblock {\em Naval Research Logistics Quarterly}, 2:83--97, 1955.

\bibitem{PushAndSwap}
R.~Luna and K.~E. Bekris.
\newblock {Push and Swap}: Fast cooperative path-finding with completeness
  guarantees.
\newblock In {\em International Joint Conference on Artificial Intelligence},
  pages 294--300, 2011.

\bibitem{MaAAMAS16}
H.~Ma and S.~Koenig.
\newblock Optimal target assignment and path finding for teams of agents.
\newblock In {\em International Conference on Autonomous Agents and Multiagent
  Systems}, pages 1144--1152, 2016.

\bibitem{MaWOMPF16}
H.~Ma, S.~Koenig, N.~Ayanian, L.~Cohen, W.~H\"onig, T.~K.~S. Kumar, T.~Uras,
  H.~Xu, C.~Tovey, and G.~Sharon.
\newblock Overview: Generalizations of multi-agent path finding to real-world
  scenarios.
\newblock In {\em IJCAI-16 Workshop on Multi-Agent Path Finding}, 2016.

\bibitem{MaAAAI17}
H.~Ma, T.~K.~S. Kumar, and S.~Koenig.
\newblock Multi-agent path finding with delay probabilities.
\newblock In {\em AAAI Conference on Artificial Intelligence}, 2017.

\bibitem{MaAAAI16}
H.~Ma, C.~Tovey, G.~Sharon, T.~K.~S. Kumar, and S.~Koenig.
\newblock Multi-agent path finding with payload transfers and the
  package-exchange robot-routing problem.
\newblock In {\em AAAI Conference on Artificial Intelligence}, pages
  3166--3173, 2016.

\bibitem{AAAI15-MacAlpine}
P.~MacAlpine, E.~Price, and P.~Stone.
\newblock {SCRAM}: Scalable collision-avoiding role assignment with
  minimal-makespan for formational positioning.
\newblock In {\em AAAI Conference on Artificial Intelligence}, pages
  2096--2102, 2015.

\bibitem{airporttug16}
R.~Morris, C.~Pasareanu, K.~Luckow, W.~Malik, H.~Ma, S.~Kumar, and S.~Koenig.
\newblock Planning, scheduling and monitoring for airport surface operations.
\newblock In {\em AAAI-16 Workshop on Planning for Hybrid Systems}, 2016.

\bibitem{DBLP:journals/ai/SharonSFS15}
G.~Sharon, R.~Stern, A.~Felner, and N.~R. Sturtevant.
\newblock Conflict-based search for optimal multi-agent pathfinding.
\newblock {\em Artificial Intelligence}, 219:40--66, 2015.

\bibitem{DBLP:journals/ai/SharonSGF13}
G.~Sharon, R.~Stern, M.~Goldenberg, and A.~Felner.
\newblock The increasing cost tree search for optimal multi-agent pathfinding.
\newblock {\em Artificial Intelligence}, 195:470--495, 2013.

\bibitem{WHCA}
D.~Silver.
\newblock Cooperative pathfinding.
\newblock In {\em Artificial Intelligence and Interactive Digital
  Entertainment}, pages 117--122, 2005.

\bibitem{ODA11}
T.~S. Standley and R.~E. Korf.
\newblock Complete algorithms for cooperative pathfinding problems.
\newblock In {\em International Joint Conference on Artificial Intelligence},
  pages 668--673, 2011.

\bibitem{WHCA06}
N.~R. Sturtevant and M.~Buro.
\newblock Improving collaborative pathfinding using map abstraction.
\newblock In {\em Artificial Intelligence and Interactive Digital
  Entertainment}, pages 80--85, 2006.

\bibitem{Surynek09}
P.~Surynek.
\newblock A novel approach to path planning for multiple robots in bi-connected
  graphs.
\newblock In {\em IEEE International Conference on Robotics and Automation},
  pages 3613--3619, 2009.

\bibitem{Surynek15}
P.~Surynek.
\newblock Reduced time-expansion graphs and goal decomposition for solving
  cooperative path finding sub-optimally.
\newblock In {\em International Joint Conference on Artificial Intelligence},
  pages 1916--1922, 2015.

\bibitem{Tovey2005}
C.~Tovey, M.~Lagoudakis, S.~Jain, and S.~Koenig.
\newblock The generation of bidding rules for auction-based robot coordination.
\newblock In F.~S. L.~Parker and A.~Schultz, editors, {\em Multi-Robot Systems.
  From Swarms to Intelligent Automata}, volume~3, chapter~1, pages 3--14.
  Springer, 2005.

\bibitem{ORCA}
J.~P. van~den Berg, J.~Snape, S.~J. Guy, and D.~Manocha.
\newblock Reciprocal collision avoidance with acceleration-velocity obstacles.
\newblock In {\em IEEE International Conference on Robotics and Automation},
  pages 3475--3482, 2011.

\bibitem{DBLP:conf/ijcai/VelosoBCR15}
M.~Veloso, J.~Biswas, B.~Coltin, and S.~Rosenthal.
\newblock {CoBots}: Robust symbiotic autonomous mobile service robots.
\newblock In {\em International Joint Conference on Artificial Intelligence},
  pages 4423--4429, 2015.

\bibitem{wagner15}
G.~Wagner.
\newblock {\em Subdimensional Expansion: A Framework for Computationally
  Tractable Multirobot Path Planning}.
\newblock PhD thesis, Carnegie Mellon University, 2015.

\bibitem{WangB11}
K.~Wang and A.~Botea.
\newblock {MAPP}: a scalable multi-agent path planning algorithm with
  tractability and completeness guarantees.
\newblock {\em Journal of Artificial Intelligence Research}, 42:55--90, 2011.

\bibitem{kiva}
P.~R. Wurman, R.~D'Andrea, and M.~Mountz.
\newblock Coordinating hundreds of cooperative, autonomous vehicles in
  warehouses.
\newblock {\em {AI} Magazine}, 29(1):9--20, 2008.

\bibitem{YuLav13ICRA}
J.~Yu and S.~M. LaValle.
\newblock Planning optimal paths for multiple robots on graphs.
\newblock In {\em IEEE International Conference on Robotics and Automation},
  pages 3612--3617, 2013.

\bibitem{ZhengIJCAI}
X.~Zheng and S.~Koenig.
\newblock K-swaps: Cooperative negotiation for solving task-allocation
  problems.
\newblock In {\em International Joint Conference on Artifical Intelligence},
  pages 373--378, 2009.

\end{thebibliography}

\end{document}